\documentclass[letterpaper, 10 pt, conference]{ieeeconf}  

\IEEEoverridecommandlockouts                              

\usepackage{alltt}
\usepackage{subcaption}
\usepackage{mathtools}
\usepackage{multirow}
\usepackage{float}

\usepackage{amssymb,amsmath,amsthm}
\newtheorem{theorem}{Theorem}
\newtheorem{mydef}{Definition}
\usepackage{pdfpages}
\usepackage[ruled,vlined]{algorithm2e}
\usepackage{multirow}
\title{\LARGE \bf
Are We On The Same Page? 
\\Hierarchical Explanation Generation for Planning Tasks in Human-Robot Teaming using Reinforcement Learning
}

\author{Mehrdad Zakershahrak, Samira Ghodratnama
\thanks{Mehrdad Zakershahrak, and Samira Ghodratnama are with the School of Computing, Informatics and Decision Systems Engineering, Arizona State University, Tempe, AZ.
        {\tt\small \{mzakersh, sghodrat\}@asu.edu} 
         {}}%
}

\begin{document}

\maketitle
\thispagestyle{empty}
\pagestyle{empty}

\begin{abstract}
Providing explanations is considered an imperative ability for an AI agent in a human-robot teaming framework.
The \textit{right} explanation provides the rationale behind an AI agent's decision-making. 
However, to maintain the human teammate's cognitive demand to comprehend the provided explanations, prior works have focused on providing explanations in a specific order or intertwining the explanation generation with plan execution.
Moreover, these approaches do not consider the degree of details required to share throughout the provided explanations. 
In this work, we argue that the agent-generated explanations, especially the complex ones, should be abstracted to be aligned with the level of details the human teammate desires to maintain the recipient's cognitive load.
Therefore, learning a hierarchical explanations model is a challenging task. 
Moreover, the agent needs to follow a consistent high-level policy to transfer the learned teammate preferences to a new scenario while lower-level detailed plans are different.
Our evaluation confirmed the process of understanding an explanation, especially a complex and detailed explanation, is hierarchical. 
The human preference that reflected this aspect corresponded exactly to creating and employing abstraction for knowledge assimilation hidden deeper in our cognitive process.  
We showed that hierarchical explanations achieved better task performance and behavior interpretability while reduced cognitive load.
These results shed light on designing explainable agents utilizing reinforcement learning and planning across various domains.
\end{abstract}

\section{Introduction}
Intelligent agents, mostly in the form of robots, continuously increase their impact on our daily life. 
An important social aspect of this phenomenon is the evolution of the robot's role as a team player in various domains.
In this regard, a robotic teammate is expected to be compatible with a human peer in a teaming interaction scheme \cite{cooke2015team}.
Therefore, the robotic teammate is desired to act comprehensible and explain the rationale behind its decision-making if necessary.

In a teaming context, explanations provide the rationale behind an individual agent's decision-making~\cite {lombrozo2006structure}, help build a shared situation awareness, and maintain trust between teammates~\cite{cooke2015team,endsley1988design, zakershahrak2020order}.
Explanations could be interactive, i.e., both explainer and the recipient have goals in the exchange, which could be conflicting. 
Their goals affect what is to be considered as an acceptable explanation~\cite{sormo2005explanation}.
Consequently, a social agent would construct explanations as excuses when a task cannot be achieved~\cite{chakraborti2017ai,gobelbecker2010coming}.

State-of-the-art work on generating explanations focuses on the factual consistency of explanations from the explainer's perspective rather than understandable explanations for the recipient of the explanations (Fig.~\ref{fig:setting})~\cite{gobelbecker2010coming,hanheide2017robot,sohrabi2011preferred}.
Moreover, there is a gap in human-robot teaming research when the teammate's goal and plan are ambiguous at the same time concerning plan explainability, predictability, and legibility all together~\cite{chakraborti2019explicability,zakershahrak2018interactive}. 
State-of-the-art work on generating explanations focuses on the factual consistency of explanations from the explainer's perspective rather than understandable explanations for the recipient of the explanations \cite{gobelbecker2010coming,hanheide2017robot,sohrabi2011preferred}.
In a teaming scheme, the human peer potentially asks the robotic agent questions along the lines of (1) \textit{Why did you do that?}, and (2) \textit{Why can't you do that?}
The answer to these questions can be complicated, depending on the social agent's high-level intention, and requires alternate model-checking. On the other hand, humans are known to have a limited attention span that makes it cognitively demanding to understand the robotic teammate's explanations.
Accordingly, providing the \textit{right} explanation might not be an adequate remedy, especially when the task is complex. 

Depending on the attention level required to fulfill the task, the goal of the human-robot interaction is to either maintain or reduce the cognitive load of the human teammate.
Overall, the attention and cognitive load have positive correlation~\cite{cheng2018exploring,berggren2013affective}.
Since our goal here is to maintain attention and focus, we aim to keep the human teammate's cognitive load throughout the interaction.
Prior work in maintaining cognitive demand in explainable planning tasks focused on (i) the influence of information order~\cite{zakershahrak2020order}; and (ii) information breakdown in an explanation generation progression framework~\cite{zakershahrak2019online}.
Progression builds later concepts on previous ones and is known to contribute to better learning and modeling the humans' preferences for information order in receiving such explanations to assist understanding in an inverse reinforcement learning setting~\cite{zakershahrak2020order}.
On the other hand, information breakdown should be intertwined with plan execution, which helps spread out the information to be explained and thus reduce humans' mental workload in highly cognitive demanding tasks~\cite{zakershahrak2019online}.

A significant similarity between planning and reinforcement learning (RL) is that both approaches provide a behavioral guideline for the agent.
While RL creates a policy based on interactions with the world, planning uses a model of the environment to create a policy.
However, a distinctive difference between these two approaches is the access to the transition and reward functions.
RL assumes visiting the states once per episode, known as \textit{irreversible sample environment}.
Simultaneously, planning assumes having access to the transition function without knowing the underlying probability, known as \textit{reversible sample environment} ~\cite{moerland2020framework}.
Therefore, RL focuses on developing an ``unknown model'' while planning starts with a ``known model''.

Recently, there is a growing focus on hierarchical approaches in reinforcement learning employing the options framework \cite{barto2003recent,konidaris2007building,sutton1999between}.
Combining planning with RL has been studied before for robotic and control mainly focused on gradient-based planning methods ~\cite{levy2017learning,grounds2005combining,illanes2020symbolic,yang2018peorl}.
These approaches provide the ability to learn an elaborated policy containing low-level actions while focusing on learning skills on a higher level.
Furthermore, explanation generation as answers to the model checking questions mentioned above, transforms the planning problem as a black-box optimization problem ~\cite{bock1984multiple}.
The domain dynamics would be treated as constraints to the (non-linear) optimization problem by knowing the objective.
Therefore, performing policy search using RL in this parameter space yields an optimal plan ~\cite{moerland2020framework}.

This work takes a step further from the previous works \cite{zakershahrak2019online, zakershahrak2020order} by proposing a generalized hierarchical framework investigating sub-explanations as more than one unit feature change, called \textit{options}. 
It implies that the robot may be allowed to explain multiple aspects of the plan simultaneously as long as they categorize under the same intent, which is useful for explaining highly correlated aspects.
Further, we investigate how abstract or in-depth an AI agent should be considered sufficient for its human peer.
As a result, we consider the detailed level of an explanation required to clarify the goal and plan ambiguity while recognizing the limited attention span of the explanations' recipient.
This implicates communicating model differences in a hierarchical explanation generation framework.
Finally, we explore composing different options to present more complex information.
The main contributions of our framework is as follows:
\begin{enumerate}
    \item It is taskable since the user can define tasks as goal conditions in the symbolic domain.
    \item It improves sample efficiency as the high-level policies can be used for transferring learning from previously learned policies.
    \item It provoke better human situation awareness, trust and higher overall performance.
\end{enumerate}
\begin{figure} 
    \centering
    \includegraphics[width = 0.7\columnwidth]{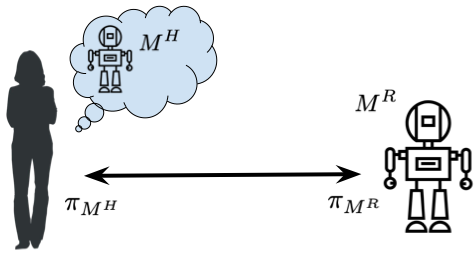}
    \caption{Explanation generation as model reconciliation \protect\cite{chakraborti1802plan}. 
    $M^R$ denotes the robot model, and $M^H$ denotes the human model used to generate her expectation of the robot's behavior ($\pi_{M^H}$).
    When the expectation does not match the robot's behavior, $\pi_{M^R}$, explanations must be generated.}
    \label{fig:setting}
    \vspace{-0.5cm}
\end{figure} 
To this end, in Section~\ref{sec:HRL} a general formulation based on goal-based Markov Decision Processes for generating progressive explanation is presented.
The proposed formulation is given the hierarchical information communication in explanation.
We propose to learn a quantification of the cognitive effort for each step as a policy in a reinforcement learning framework~\cite{sutton1990integrated,sutton1999between, konidaris2018skills}. 
We test the following hypothesis in Section~\ref{sec:eval}:
\begin{itemize}
    \item H1: Our method learns about humans' preferences in receiving abstract or in-depth explanations.
    \item H2: Assuming cognitive load correlation with replanning cost, hierarchical explanations maintain cognitive load and improve task performance.
\end{itemize}
The subjective post-study results validated H1.
Furthermore, a comparison with two baseline methods validated H2. 
We show that the hierarchy of explanations corresponded well to the ``abstraction'' of the curve on domain features result in higher human situation awareness, trust, and overall task performance.
These hypotheses confirm the impact of hierarchical explanations on human situation awareness, trust, and higher overall performance, which increases agent transparency ~\cite{de2014design,lyons2014transparency,mercado2016intelligent,ososky2014determinants,miller2021trust}.
We evaluated these hypotheses on a scavenger-hunt domain described in Section~\ref{sec:domain}. 

\section{Background}
\subsection{Planning}
A planning task is defined using a tuple $\mathcal{P}=\langle \mathcal{F}, A, I, G\rangle$, where $\mathcal{F}$ is the set of predicates used to specify the state and A the set of actions used to update the state. 
Each action $a \in A$ changes the state of the world by adding or deleting predicates.
$a = (pre(a), eff^{+}(a), eff^{-}(a), c_a)$; where $pre(a)$ denotes the preconditions of the action $a$, and $eff^{+}(a), eff^{-}(a)$ indicate add and delete effects, respectively, and $c_a$ is the cost of the action.

A goal-based Markov Decision Process (MDP) is defined as a tuple $\mathbf{M}=\langle S, A, T, r, \gamma, G \rangle$, where $S$ represents a unique state in which a set of predicates, $f \subseteq \mathcal{F}$, is available and $\{f\}\backslash \mathcal{F}$ is unavailable.
An action $a \subseteq A$, is legible in state $S$ when the set of preconditions of action $a$ is available in $f$, that is $pre(a)\in f$. 
Similarly, the transition function $T(S,a, S^\prime)$, is transferring the current state to the state $S^\prime$, where $S^\prime = \{f \cup eff^{+}(a)\} \backslash eff^{-}(a)$.
The domain dynamics is represented as the transition function $T$ that determines the probability of transitioning into state $s'$ when taking an action $a$ in state $s$ (i.e., $P(s'|s,a)$).
$R$ is the reward function and the goal of the agent is to maximize the expected cumulative reward.
$\gamma$ is the discount factor that encodes the agent's preference of current rewards over future rewards.
$G$ is a set of goal states where for all $g \in G$, $T (g, a, g) = 1, R(g, a, g) = 0, \forall a \in A$. 
The reward is represented by R(s,a,$s^\prime)= 1$, if $s^\prime \in G$ and R(s,a,$s^\prime$) = 0 otherwise.

\begin{mydef}[Plausible Plan] \label{def:plan}
Given $\mathbf{M}=\langle S, A, T, r, \gamma, G \rangle$, and an initial state $i \in I$, a plausible plan is a set of sequence of actions, $\rho = \{a_1, a_2,...,a_n\}$ where it reaches the goal, $g \in G$, deterministically. 
\end{mydef}

\begin{theorem}
A plausible plan $\pi$ maximizes the probability of reaching a goal state.
\end{theorem}

\begin{proof}
We calculate the probability of a plausible plan by:
\begin{equation} \label{eq:plan}
\begin{split}
    P(g|i,\pi) = &\\
    P(s_2|i,a_1)\cdot P(s_3|s_2,a_2)\cdot ... \cdot P(g|s_n,a_n)=1,
    i.e.\\
    P(s_2|i,a_1) = P(s_3|s_2,a_2) = ... = P(g|s_n,a_n) = 1
\end{split}
\end{equation}
Therefore, the probability of a plausible plan is $1$. 
\end{proof}

Cost of a plausible plan, $J(s)$, is defined by:
\begin{equation} \label{eq:j}
J(s) = \sum_{s^\prime \in S}P(s^\prime|s,a)[c_a + \gamma J^*(s^\prime)]
\end{equation}
where $J(s)$ is the discounted cost of performing actions in a plan.

\begin{mydef} [Optimal Plan in Goal-based MDP] \label{def:oplan}
The optimal plan $\rho^*$ is a plausible plan that has minimum expected cost from I to G.
\end{mydef}
Therefore, the cost of an optimal plan is calculated by:
\begin{equation} \label{eq:jstar}
J^*(s) = min_{a \in A}\sum_{s^\prime \in S}P(s^\prime|s,a)[c_a + \gamma J^*(s^\prime)]
\end{equation}
We get the states belong to the optimal policy by applying the $\rho^*$ in the transition function:
\begin{equation} \label{eq:sstar}
    s_{\rho^*} = \bigcup_{i=1}^{n}\{s_i|\ T(s_i,a_i, s_{i+1})=1\}
\end{equation}
where $s_1 $ is initial state, and $s_{n+1} = g$.

\begin{mydef}[Optimal policy] Given $\rho^*$ for an initial state $i \in I$ and a goal state $g \in G$ in $\mathbf{M}$, the optimal policy function is defined as $\pi^*(s) = a$, $a \in \rho^*$, $\forall s \in S$.
\end{mydef}
Using the transition function we can calculate the probability of the optimal policy:
\begin{equation}
    \begin{split}
        P(\pi^*|\hspace{1mm}\cdot\hspace{1mm})= T(i,a_1, s_2)\cdot \prod_{i=2}^{n-1}T(s_i,a_i, s_{i+1})\cdot T(s_n,a_n, g),
        \\ \forall s \in S, \forall a \in \rho^*
    \end{split}
\end{equation}
The maximum discounted reward of an optimal policy is: 
\begin{equation} \label{eq:v}
V^*(s) = max_{a \in \mathcal{A}}\sum_{s^\prime \in S}P(s^\prime|s,a)[R(s,a,s^\prime) + \gamma V^*(s^\prime)]
\end{equation}

\begin{theorem} \label{th:two}
An optimal policy $\pi^*$ of a Goal-based MDP, $\mathbf{M}$, is plausible and maximizes the discounted reward of reaching a goal state.
\end{theorem}

\begin{proof}
Let's assume that the optimal policy $\pi^*$ is not plausible. Therefore, $\exists s\in s_{\rho^*}$ s.t. $\forall a \in A, P(g|s,a) = 0$. This contradicts with the Equation~\ref{eq:plan}. Therefore, an optimal policy is plausible.
On the other hand, let us assume that the optimal policy $\pi^*$ does not maximize the reward of reaching the goal state. As a result, $\exists s\in s_{\rho^*}$ s.t. $\forall a \in A, P(g|s,a) < P(g|s^\prime,a)$ where $s^\prime \notin s_{\rho^*}$. This also contradicts with Equation~\ref{eq:sstar}. Consequently, an optimal policy maximizes the discounted reward of reaching a goal state.
\end{proof} 

\begin{theorem} \label{th:three}
An optimal plan $\pi^*$ has the optimal number of steps between the initial state $i$ and a goal state $g$, assuming $\gamma < 1$.
\end{theorem}
\begin{proof}
Following Theorem~\ref{th:two}, we know $\pi^*$ maximizes the discounted reward of reaching the goal state.
Also, from the definition of goal-based MDP, we know that reaching the goal is the only non-zero reward compared to other non-goal states' reward. 
By applying Equation \ref{eq:v} for an optimal policy, the reward is maximizing when we reach the goal state. Therefore, assuming $\gamma < 1$, the number of steps between $i$ and $g$ is minimized.
\end{proof}

\subsection{Reinforcement Learning (RL)}
We define the problem of RL on a goal-based MDP, $\mathbf{M}=\langle S, A, T, r, \gamma, G \rangle$, as finding an optimal policy $\pi^*$ that maximizes the expected discounted cumulative reward for all of the states in the optimal policy, given an initial state $i$.
\begin{equation}
    V_{\pi}(s) = \mathbb{E} \Big[ \sum _{t = 0}^{\infty} \gamma^t r_t |s_0 = i\Big]
\end{equation}
Initially, the agent starts with a random policy, observes the current state, and chooses an action based on the policy.
Then, based on the sampled probability, $p(s^\prime|s,a$), and using the reward, the agent updates its current policy. 
In our work, we use \textit{SARSA} to select the next action and improve the current policy. 
\section{Hierarchical RL using Goal-based MDP} \label{sec:HRL}
Traditional RL approaches require lots of training data to converge, which leads to scaling issues when are applied to large-scale problems.
The popular approach is to use \textit{options} which allows defining high-level actions.
Furthermore, the learned options can be put in use in different scenarios as long as the options' semantics, i.e., their intent, remains consistent with the learned domain.
Therefore, the action level states can change while the higher-level options remain the same across various tasks.

An option is defined as a tuple $o = \langle I_o, \pi_o, r_o, T_o\rangle$, where $\pi_o$, $r_o$ and $T_o \subseteq S$ are the option's policy, reward and the set of termination states for the set of initial states $I_o$ respectively.
Each option, $o \in \mathcal{O}$ is a high-level action, called intent-level, consisting of a set of action-level states in planning. 
Hence, following each option amends multiple predicates in the planning domain. 
The goal states determine termination of the options and their legibility \cite{konidaris2007building}.
Learning policies employing options provides the opportunity to achieve high-level behaviors while defining low-level states' distribution for their termination ~\cite{sutton1999between}. 
The q-function is defined for one option as:
\begin{equation}
    q(s,a) = r(s,a,s^{\prime}) + \gamma \cdot q(s^{\prime}, a^{\prime}), \forall s \in \pi_o
\end{equation}
similarly,
\begin{equation}
    q^* (s,a) = max \mkern4mu q(s,a), \forall s \in \pi_o
\end{equation}
The option's terminations states, $T_o$, are considered sub-goals or option goals.
Therefore, $q(s,a) = 1$ where $s \in T_o$, and $q(s^{\prime}, a^{\prime})$ otherwise, and $s, s^{\prime} \in \pi_o$. 
This means the sub-goals' reward is the only non-zero reward compared to the non-sub-goal states' reward.
This representation of reward is well-positioned in the literature \cite{koenig1993complexity,sutton1990integrated,peng1993efficient,white1992handbook}.

If Q-learning initializes the states with zero, the agent will perform a random walk. 
In the worst-case scenario, the agent explores $2^n$ states. 
Cooperative Reinforcement learning uses directed exploration since the agent is aware of the goal and decreases the worst-case complexity to polynomial w.r.t the number of the states required to be explored.
Therefore, an \textit{admissible} q-function is defined as:
 \begin{equation}
     q^*(s,a) = argmax  \mkern4mu q_{a^\prime \in \rho^*}(s, a^\prime)
 \end{equation}
Consequently, given a set of options, the on-policy method SARSA~\cite{rummery1994line} updates the intent-level value function, the $Q$-function, as follows:
\begin{equation}
    Q(s,o) = Q(s,o) + \alpha\big( r_{o} + \gamma Q(s^\prime,o^\prime) - Q(s,o) \big), \forall o \in \mathcal{O}
\end{equation}
Where $\alpha$ is the learning rate and $\gamma$ is the discount factor.
The on-policy SARSA for estimating Q is presented at Algorithm~\ref{alg:q}.

\begin{theorem} \label{th:five}
An optimal option's policy $\pi_o^*$ has the optimal number of steps between the initial state i and a goal state g $\in T_o$, assuming $\gamma < 1$.
\end{theorem}

\begin{proof}
Following proof of Theorem~\ref{th:three}, since the reward of reaching the goal is comparatively much more significant than other states' reward, by applying Equation~\ref{eq:v} for an optimal policy, the reward is maximized when we reach the goal state. 
Therefore, assuming $\gamma < 1$, the number of steps between $i$ and $g$ is minimized.
\end{proof}

\subsection{Mindset Adaptation in Hierarchical RL}
A mindset is defined using a tuple $\mathcal{M} = \langle \mathbf{M}, \mathcal{O} \rangle$, where $\mathcal{O}$ is set of options on the abstract level given to the robot, and $\mathbf{M}$ is a goal-based MDP model of the environment. 
To capture the mindset changes, the model function $\Lambda: \mathcal{M} \rightarrow 2^F$ is defined to convert a mindset to a set of mindset features \cite{chakraborti1802plan}, where $\mathcal{M}$ is the mindset space and $F$ the feature space.
In this way, one mindset can be adapted to another mindset using editing functions that changes multiple features at a time. 
The set of feature changes is denoted as $\Delta(M_1, M_2)$ and the distance between two mindsets as the number of such feature changes is denoted as $\delta(M_1, M_2)$.
In this work, we assume that the mindset is defined in PDDL \cite{fox2003pddl2}, an extension of STRIPS~\cite{fikes1971strips} as explained in the planning section under related work.

\begin{mydef} [Mindset Adaptation]
We define the mindset adaptation as a tuple: $\langle(\mathcal{M^R, M^H}),\Lambda \rangle$, where $\Lambda$ modifies the goal-based MDP of the human ($\mathcal{M^H}$) to robot ($\mathcal{M^R}$) using $\mathcal{O}$ s.t. the robot's optimal plausible plan $\rho_{R}$ becomes optimally plausible for human as well with respect to Definition \ref{def:oplan}.
\end{mydef} 

At each adaptation step, the robot introduces an option, $o \in \mathcal{O}$, uses $\Lambda$ to modify the features present at the current state of the mindset and eventually gets to $T_o$. 
Therefore, the adaptation process solves an optimization problem utilizing reinforcement learning, providing the robot's intent on the intent-level while solves a planning problem on the action-level.
This process is illustrated in Fig. \ref{fig:hierarchy}.

Mindset adaptation introduces new partitioning on the action-level by providing a new option on the intent level.
Each new tree-like structure is an option on the intent level, and states on the action level are similar to belief distribution. 
Therefore, the cost of adaptation, w.r.t. Equation \ref{eq:j}, is defined as $\Gamma$ in the following Equation:
\begin{equation}
    \Gamma(J_{M^H}, J_{M^R}) = |\underset{s \in M^H}{J(s)} - \underset{\widehat{s} \in M^R}{J(\widehat{s})}|
\end{equation}
As a result, introducing new options can be viewed as changing the posterior belief distribution, which is different from a human's prior belief distribution.
Consequently, performing the mindset adaptation, the robot does not explain the details of planning proactively to maintain the human teammate's cognitive load.
The optimal cost of adaptation, w.r.t. Equation \ref{eq:jstar}, is calculated in the following Equation:
\begin{equation} \label{eq:gammastar}
    \Gamma^* = \Gamma(J^*_{M^H}, J^*_{M^R}) = |\underset{s \in M^H}{J^*(s)} - \underset{\widehat{s} \in M^R}{J^*(\widehat{s})}|
\end{equation}

\begin{mydef} [Concise Mindset Adaptation] \label{def:concise}
We define concise mindset adaptation as a tuple $\langle(\mathcal{M^R, M^H}),\Lambda\rangle$, where $\Lambda$ has the optimal adaptation cost with respect to Equation \ref{eq:gammastar}.
\end{mydef}

\begin{figure*} [htbp]
    \centering
    \includegraphics[width=0.8\textwidth]{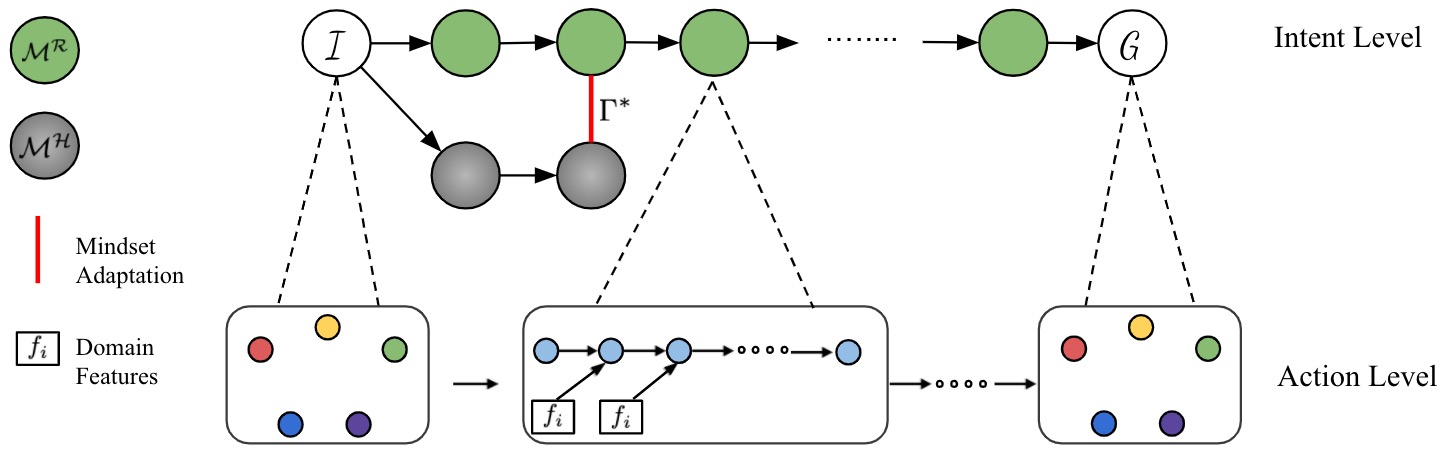}
    \caption{Hierarchical explanation generation.
    The agent learns/employs options on the intent-level while at the action-level, explanations are features added/deleted to each state.
    On the intent-level, explanations about options are intents, while explanations about action-level provide transparency.
    Intent-level options are correlated with action-level states.
    While the agent finds an optimal policy on the intent-level, it provides planning on the action-level.}
    \label{fig:hierarchy}
    \vspace{-0.5cm}
\end{figure*}
Applying the mindset adaptation, the agent creates a hierarchical representation of the domain \cite{konidaris2018skills}.
Generating explanations employing this representation is helpful toward maintaining the cognitive load of the human teammate. 
Throughout the mindset adaptation, there are options that the agent requires to explain in more detail by going down the hierarchy.
Similarly, there are options that the human teammate would be notified about the high-level option, and no further details are required. 
\subsection{Learn Abstract Explanations from Experience}
The agent constructed a planning representation by searching through possible combinations of action-level states within each classification of the option's precondition mask.
 To compute the policy's probability at the intent-level, correlated with action-level states, we calculate the conjunction of the options, distributional beliefs. 
However, without the loss of generality, we use deterministic options to be compatible with PDDL ~\cite{fox2003pddl2}. 
Therefore, in this paper, we used deterministic policies reinforcement learning to calculate the optimal policy.

\begin{mydef} [Concise Explanation Generation by Mindset Adaptation]
Given $(\mathcal{M^R, M^H})$, the objective of concise explanation generation is to find the $\Lambda$, where the cardinality of changes required, $|\Lambda| \leq \Delta(M_1,M_2)$, are minimum subject to the Definition \ref{def:concise}. 
\end{mydef}

Following Equation~\ref{eq:gammastar} and consecutive to Theorem \ref{th:two}, the concise mindset adaptation process has the minimum human plan adaptation cost while maximizes the discounted reward of reaching a goal state.
The expected reward model of the optimal policy, at intent-level, is learned using the human's \textit{expert} traces as input while assuming that the concise mindset adaptation provides minimum cost to amend the human original plan. 
The mindset adaptation process to generate the concise explanations ($\Lambda$) are presented at Algorithm \ref{alg:mindset}. 

\section{Evaluation} \label{sec:eval}
We evaluated our approach by conducting human-subject studies using Amazon Mechanical Turk (MTurk) in the scavenger-hunt domain.
We design this domain to create complex situations that require subjects to invest a significant amount of cognitive effort quickly. 

\begin{algorithm}[!b]
\SetAlgoLined
 \SetKwInOut{Input}{input}
 \SetKwInOut{Output}{output}
 \Input{$\langle\mathcal{M^H},\mathcal{M^R}, I, G\rangle$}
 \Output{$\Lambda$}
 Compute $\Delta(\mathcal{M^H},\mathcal{M^R})$ as the difference between the two mindsets\;
 Compute $\pi^*_H$ using $\mathcal{M^H}$\;
 Compute $\pi^*$ using $\mathcal{M^R}$\;
 initialize $\Lambda, q = \{\}$\;
 \While{$\pi^*_H \neq \pi^*$}{
        \For{each $\delta \in \Delta$}{
            $q$ $\leftarrow$ $Q(\delta, \epsilon)$\;
        }
        $\mathcal{M^H}^\prime$ $\leftarrow$ Modify $\mathcal{M^H}$ using $argmax(q)$\;
        $\Lambda.append()$ $\leftarrow$ $\forall o \in argmax(q)$\;
        Update $\pi^*_H$ using $\mathcal{M^H}^\prime$\;
 }
 
 return $\Lambda$\;
 \caption{Mindset Adaptation}
 \label{alg:mindset}
\end{algorithm}

\begin{algorithm}[!b]
\SetAlgoLined
 \SetKwInOut{Input}{input}
 \SetKwInOut{Output}{output}
 \Input{$\delta,\epsilon $}
 \Output{$Q(s,o)$, $\forall o \in \delta$}
 initialize $Q(s,o)$ arbitrarily\;
 \For{each episode}{
    initialize $s$ as i $\in I$\;
    $p_1$ $\leftarrow$ Random(\textit{seed})\;
    \While{$s \notin G$}{
        \eIf{$(p_1<\epsilon)$}{
            $o$ $\leftarrow$ Maximum value using $\pi$ of $Q(s,o)$\;
            }
            {$o$ $\leftarrow$ Choose randomly\;}
            \For{each episode}{
                $r$, $s^\prime$ $\leftarrow$ Perform(o)\;
                $p_2$ $\leftarrow$ Random(\textit{seed})\;
                \eIf{$(p_2<\epsilon)$}{
                    $o^\prime$ $\leftarrow$ Maximum value using $\pi$ of $Q(s^\prime,o^\prime)$\;}
                    {$o^\prime$ $\leftarrow$ Choose randomly\;}
                    $s \leftarrow s^\prime$\;
                    $o \leftarrow a^\prime$\;
        }
    }
 }
 return $Q(s,o)$\;
 \caption{SARSA for mindset adaptation}
 \label{alg:q}
\end{algorithm}

\begin{figure*} [hbtp]
\centering
\includegraphics[width=0.9\textwidth]{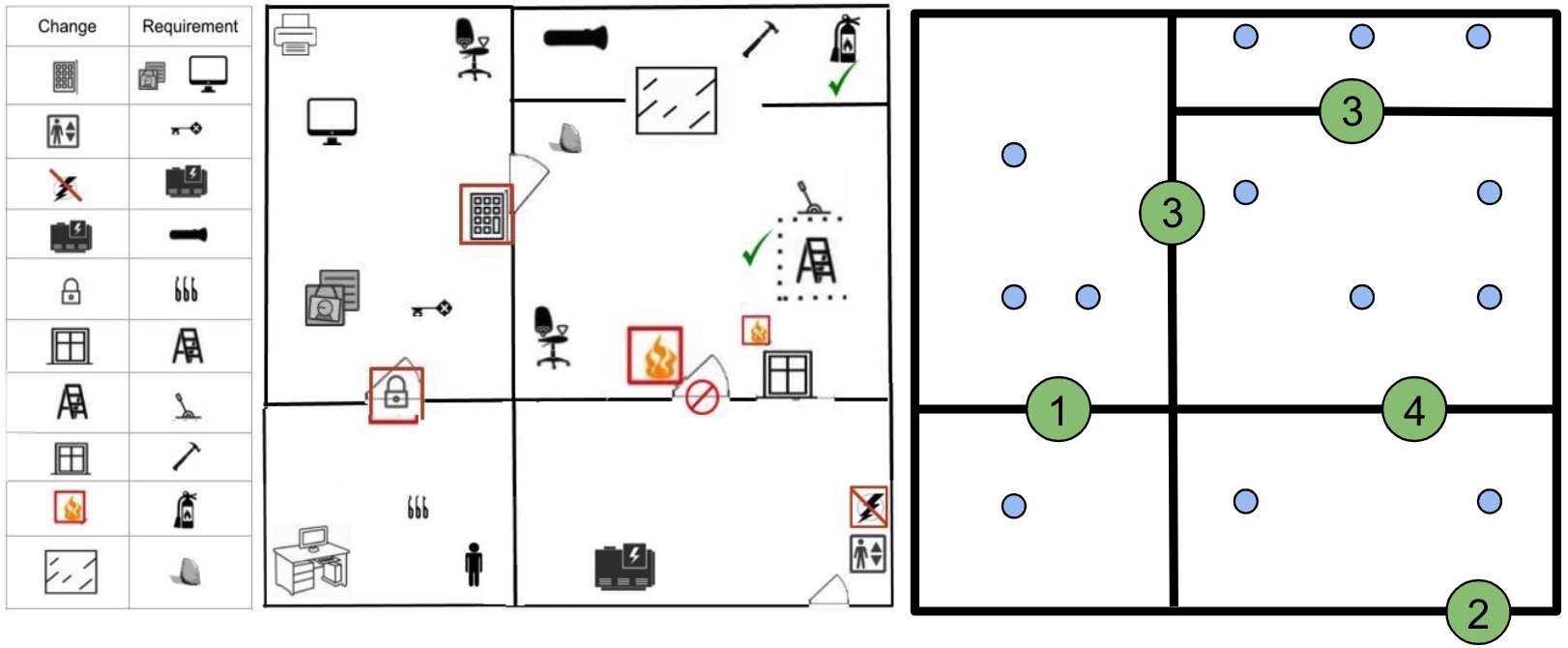}
\caption{Left:Illustration of the scavenger-hunt domain. Right: The hierarchical planning task representation of the domain.
The green circles, intent-level options, are local optimas for the policy based on SARSA, and the small blue circles are action-level states for each option. 
The numbers in each local optima represent the number of action-level explanations the agent has to provide, depending on the scenario, upon learning an option.
}
\vspace{-0.5cm}
\label{fig:scavenge}
\end{figure*}

\subsection{Scavenger-Hunt Domain} \label{sec:domain}
The domain portrays a damaged office building after an earthquake, showed by a floor-plan in Fig.~\ref{fig:scavenge}. 
Typically, the human uses the doors to exit each room and eventually exit the building via the elevator from his office.
However, an earthquake may interrupt the human's original path in different ways.
The first response team sent an autonomous robot to the damaged building to find and help the trapped human navigate through the building. 
The robot's goal is to explain to the human whenever the robot's plan (optimal plan) becomes less interpretable according to the human's model of the environment.
Therefore, the robot provides its intent to the human to verify the legibility of the synthesized plan.
At any step, the robot can explain multiple correlated changes illustrated in red boxes, as an intent-level explanation or a unique feature change, action-level explanation.
For example, the following is an illustration of intent level explanation \textit{power-out} and action level explanations such as \textit{glass-broken, fire, or has-fireExt}. 
As Fig.~\ref{fig:example} shows, for simplicity, we use only boolean variables as predicates.

\footnotesize{
\begin{alltt}
(:action activate_pin :parameters (?room - roomTWO)
:pre (and (at ?room )...\textbf{(power-out)})
:$eff^{+}$       (locked-door ?room))

(:action elevator_out_of_service \\
:parameters (?room -roomFive)
:pre (and (at ?room )...\textbf{(power-out)})
:$eff^{+}$       (door-jam ?room))

(:action get_fireExt :parameters (?room - roomTwo) 
:precondition (and (at ?room) \textbf{(glass_broken) (fire)})
:effect (and \textbf{(has_fireExt)} ))

(:action putout_fire :parameters (?room - roomFour)
:pre (and (at ?room )\textbf{(has_fireEXT)(power-out)})
:$eff^{+}$ (no_fire))
\end{alltt}
}
\normalsize
The participants are given a map that contains all the possible changes. 
These are also shown on the left side of Fig. \ref{fig:scavenge}.
A total of $10$ possible changes has changed the trapped human's domain, which resulted in $13$ changes in the action-level, classified in $5$ intent level options.
Thus, the action-level space has $2^{10}$ states, while the intent level has $2^5$ states.
For any given scenario, only a few selected changes will be present. 

Fig.~\ref{fig:example} illustrates an explanation hierarchy that links the causality of high-level intents to the detailed low-level policy that is different from a human's model due to the discrepancies.
The robot starts explaining the discrepancies utilizing a top-down approach and stops going to the details when the human can understand why the robotic agent performs a different action from her expectation. 

\begin{figure}
    \centering
    \includegraphics[width=0.9\columnwidth]{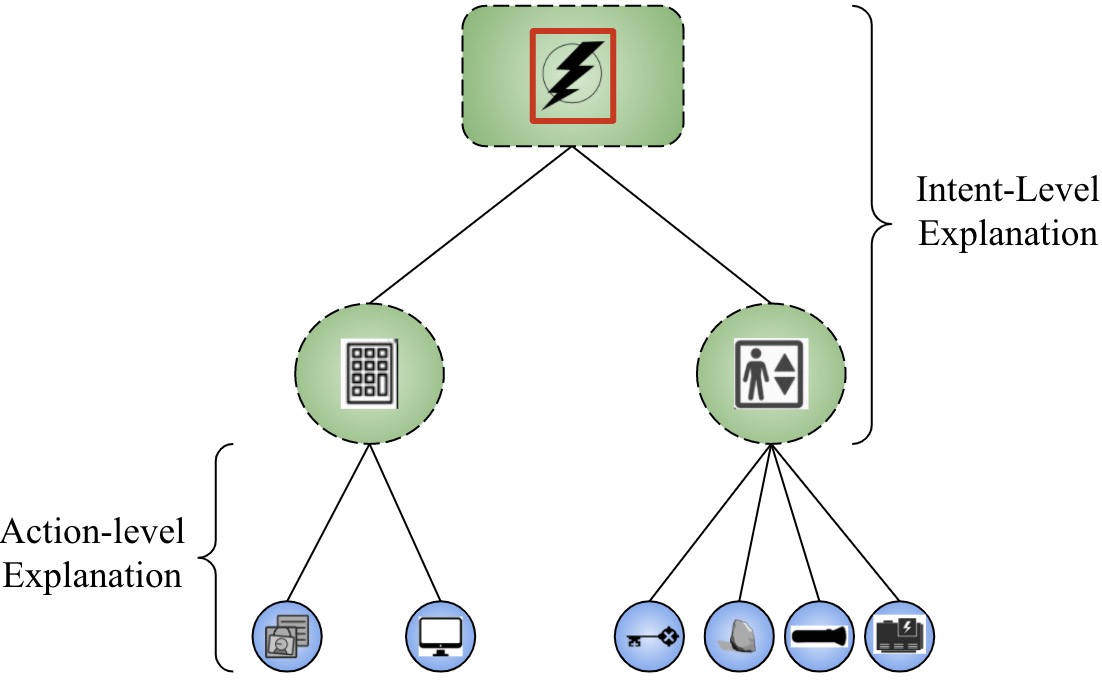}
    \caption{The explanation links the causality of high-level intent to the detailed low-level policy. The power outage caused the PIN pad in the second room to be activated and the elevator to be disabled. 
    For the elevator, a key must be put in the emergency power generator in a circuit that activates the elevator.
    } 
    \label{fig:example}
    \vspace{-0.5cm}
\end{figure}

\subsubsection{Experiment Design}
First, we explained the scavenger-hunt domain to participants and tasked them to play the trapped person's role. 
Second, we asked the participants to determine the plan's next action after each robot's explanation and write down the robot's intent. 
This step was designed to encourage the participants to clearly understand the situation such that the robot aimed to unfold the plan. 
The explanations were provided in plain English language, and robot's actions were depicted using GIF images in a 2D simulation, as shown in Fig.~\ref{fig:scavenge}, to the participants.
Then, we asked the participants to explain if the robot is performing according to their expectations. 
To ensure that the subject must question some actions to perform well, we purposely inserted random actions after some of the explanations.
Moreover, to ensure the data's quality, we implemented a sanity check question to ensure the participants understood the task.
We removed the responses with wrong answers to the sanity questions or took them over four minutes to finish the task. 

\subsection{Results}
We conducted a survey using Qualtrics and recruited $68$ human subjects using Mturk, with a HIT acceptance rate of $99\%$.
After sifting out the responses described in the previous section, we used $54$ responses over $10$ scenarios. 
We compared the outputs of our explanation generation algorithm (H-RL) based on the reward policy learned by hierarchical RL algorithm with the subjects' responses for two other baselines: Online Explanation Generation (OEG) ~\cite{zakershahrak2019online} and Progressive Explanation Generation (PEG) ~\cite{zakershahrak2020order}.
These baselines are state of the art for maintaining cognitive demand for explanation generation in planning tasks in a human-robot teaming scheme. 
OEG focuses on dividing the explanation generation process and tangling them with plan execution concerning plan prefix, plan optimality, and the clarity of the next executed action~\cite{zakershahrak2019online}.
PEG focuses on the comprehension progression of a sequence of explanations as a cognitive demand to amend the original plan~\cite{zakershahrak2020order}. 

Our evaluation aimed to analyze whether we can learn the human preferences from training scenarios and applying them to the testing scenarios (H1).
Our method's accuracy was $94.4\%$: our approach successfully matched $51$ out of the $54$ human responses across $10$ scenarios as illustrated in Table~\ref{tab:accuracy}. 
This result showed that humans indeed had particular preferences for information hierarchy in such situations and that our method could capture these preferences.
Moreover, this Table reveals that sharing all of the detailed information, even in a preferred progressive sequence (PEG), is not as useful as abstraction since the accuracy is considerably lower in PEG compared to the two other methods. 
Furthermore, although dividing the explanations and entwining them with plan execution helps increase the accuracy in OEG, the results confirm that our approach's effectiveness is mainly due to the hierarchical context of the explanations.
These results verify H1.

\begin{figure}
    \centering
    \includegraphics[width=0.9\columnwidth]{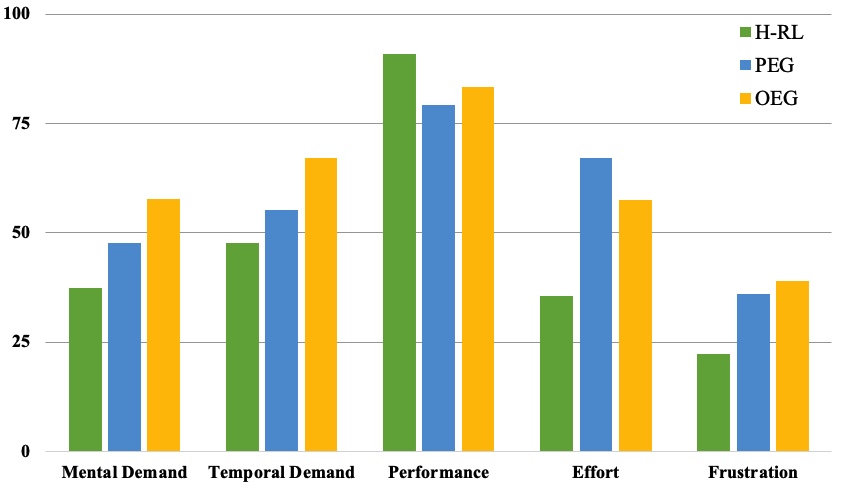}
    \caption{Comparison of the results of all of TLX categories across different settings}
    \label{fig:tlx}
\end{figure}

\begin{figure}
    \centering
    \includegraphics[width=0.45\columnwidth]{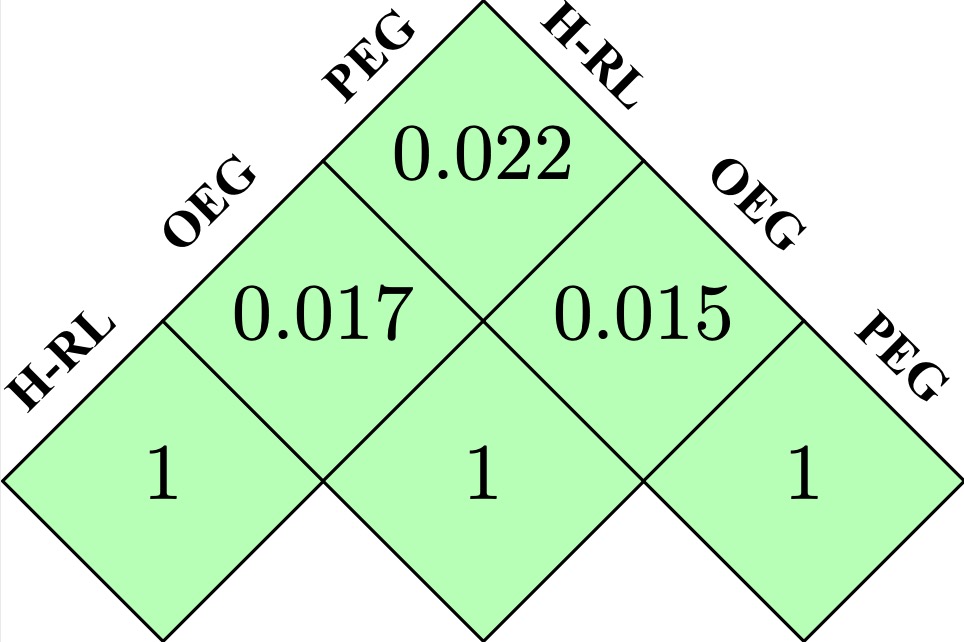}
    \caption{p-values across different approaches on the mental workload}
    \label{fig:p-value}
    \vspace{-0.5cm}
\end{figure}


Table~\ref{tab:scenario} showed the comparison results of the time taken to calculate explanations using different approaches in the scavenger-hunt domain across different scenarios.
As shown in this Table, the time for H-RL is considerably lower than the other two methods.
One reason is the deterministic policies' choice for RL helped to reduce the search space. 
Moreover, this Table illustrates that the total number of H-RL explanations is lower than other approaches, which contributes to lowering the cognitive demand required to understand the changes (H2).

The subjective measures in Fig.~\ref{fig:tlx} reaffirm the conclusions.
H-RL has the best performance and the lowest mental and temporal demand.
Due to intertwining explanations with plan execution, OEG is expected to create more temporal demand.
The $p$-values for the subjective measures are presented in Fig.~\ref{fig:p-value}. 
The results indicate statistically significant differences between H-RL, OEG, and PEG. 
Thus, these subjective results with results of Table \ref{tab:scenario} confirm H2.

\begin{table}
    \centering
    \begin{tabular}{l|c|c|c|c}
    \hline
    ~               & OEG & PEG   & H-RL & Random \\ \hline
    Accuracy        & 0.866 & 0.775 & \textbf{0.944}  & ~   \\ \hline
    \# Actions & 4.714 & 7.161 & \textbf{3.170}   & 3/36   \\ \hline
    \end{tabular}
    \caption{Objective performance in terms of plan accuracy and number of questionable actions based on the subjects' feedback for the three settings.
    The ground truth for questionable actions is $3$ out of $36$ in total.}
    \label{tab:accuracy}
\end{table}
\begin{table}[h]
\centering
\renewcommand{\arraystretch}{1.2}
    \begin{tabular}{l|c|c|c|c|c|c}
    \hline
    Pr. & \multicolumn{2}{c|}{H-RL} & \multicolumn{2}{c|}{OEG} & \multicolumn{2}{c}{PEG} \\
    ~   & $|E|$ & Time & $|E|$ & Time & $|E|$ & Time \\ \hline \hline
    P1   &  3  & 18.2   & 7  & 43.1 & 7  & 33.7      \\ 
    P2   &  5   & 21.8  & 8  & 43.6 & 7  & 35.4      \\ 
    P3   &  5   & 23.7  & 12 & 45.3 & 12 & 38.5      \\ 
    P4   &  5   & 28.0  & 11 & 44.8 & 11 & 36.2      \\
    P5   &  4   & 25.4  & 10 & 44.4 & 9  & 35.9         \\
    P6   &  5   & 34.3  & 10 & 44.5 & 10 & 36.0         \\
    P7   &  5   & 30.5  & 9  & 44.1 & 9  & 36.1         \\
    P8   &  4   & 26.2  & 10 & 44.3 & 9  & 36.0         \\
    P9   &  3   & 19.7  & 8  & 39.4 & 8  & 35.7         \\
    P10   & 4   & 25.9  & 8  & 39.5 & 8  & 35.5         \\ \hline \hline
    Average & \textbf{4.3} & \textbf{25.37}& 9.3& 43.3& 9& 35.9 \\
    \hline
    \end{tabular}
    \caption{Comparison of explanation size, and time (in seconds) taken to generate the explanations using the different methods across $10$ scenarios of the scavenger-hunt domain.}
    \label{tab:scenario}
    \vspace{-0.6cm}
\end{table}

\section{Conclusion \& Future Work}
In this paper, we studied hierarchical explanation generation employing reinforcement learning to maintain the cognitive load.
We took a step further from the prior work by considering the explanation abstraction level while considering the underlying cognitive effort required for the human teammate to understand the agent's intent.
This results in a general framework for hierarchical explanation generation. 
One interesting observation conclusion was that making an explanation is an incremental cognitive process based on the shared information's details' scale.
To address the challenge with modeling human preferences of the information order, we adopted a goal-based MDP. 
We applied RL to learn the explanation process as a policy based on traces. 
Our first contribution is to show that humans indeed demonstrate preferences for the information ranking to scale based on its detail level. We can indeed learn about such preferences using our framework. 
This verified H1.
Results from this domain validated H2. 
Finally, we showed that H-RL did improve the task performance and reduce cognitive load. 

One interesting future direction is to generalize the MDP model and use Boltzmann machines as an energy-based model instead since MDP is quite restrictive for modeling human cognition. 
Moreover, it enables the use of continuous stochastic cognition distributions over all possible plans.
This is well suited for probabilistic planning and a more realistic assumption for human cognition, which correlates with the predictability and legibility of different plans.
As a result, they create more robust plans that are less likely to overfit.
\bibliographystyle{IEEEtran}
\bibliography{reference}
\end{document}